\documentclass{article}

\usepackage{microtype}
\usepackage{graphicx}
\usepackage{subcaption}
\usepackage{booktabs} 

\usepackage{hyperref}

\usepackage[preprint]{icml2026}

\usepackage{amsmath}
\usepackage{amssymb}
\usepackage{mathtools}
\usepackage{amsthm}

\usepackage[capitalize,noabbrev]{cleveref}
\usepackage{wrapfig}

\newcommand{\mc}{\mathcal}
\newcommand{\mb}{\mathbb}

\DeclareMathOperator*{\argmax}{argmax}

\theoremstyle{plain}
\newtheorem{theorem}{Theorem}[section]

\newtheorem{lemma}[theorem]{Lemma}
\newtheorem{corollary}[theorem]{Corollary}
\theoremstyle{definition}
\newtheorem{definition}[theorem]{Definition}
\newtheorem{assumption}[theorem]{Assumption}
\theoremstyle{remark}
\newtheorem{remark}[theorem]{Remark}

\usepackage[textsize=tiny]{todonotes}

\icmltitlerunning{Intrinsic Reward Policy Optimization for Sparse-Reward Environments}

\begin{document}

\twocolumn[
\icmltitle{Intrinsic Reward Policy Optimization for Sparse-Reward Environments}




\begin{icmlauthorlist}
\icmlauthor{Minjae Cho}{uiuc}
\icmlauthor{Huy T. Tran}{uiuc}
\end{icmlauthorlist}

\icmlaffiliation{uiuc}{The Grainger College of Engineering, University of Illinois Urbana-Champaign, Urbana, USA}

\icmlcorrespondingauthor{Minjae Cho}{minjae5@illinois.edu}

\icmlkeywords{Reinforcement learning, hierarchical RL, intrinsic rewards}

\vskip 0.3in
]


\printAffiliationsAndNotice{}  

\begin{abstract}
    Exploration is essential in reinforcement learning as an agent relies on trial and error to learn an optimal policy.
    However, when rewards are sparse, naive exploration strategies, like noise injection, are often insufficient.
    Intrinsic rewards can also provide principled guidance for exploration by, for example, combining them with extrinsic rewards to optimize a policy or using them to train subpolicies for hierarchical learning.
    However, the former approach suffers from unstable credit assignment,
    while the latter exhibits sample inefficiency and sub-optimality.
    We propose a policy optimization framework that leverages multiple intrinsic rewards to directly optimize a policy for an extrinsic reward without pretraining subpolicies.
    Our algorithm---intrinsic reward policy optimization (IRPO)---achieves this by using a surrogate policy gradient that provides a more informative learning signal than the true gradient in sparse-reward environments.
    We demonstrate that IRPO improves performance and sample efficiency relative to baselines in discrete and continuous environments, and formally analyze the optimization problem solved by IRPO.
    Our code is available at \url{https://github.com/Mgineer117/IRPO}.
\end{abstract}

\begin{figure*}
    \centering
    \includegraphics[width=0.9\linewidth]{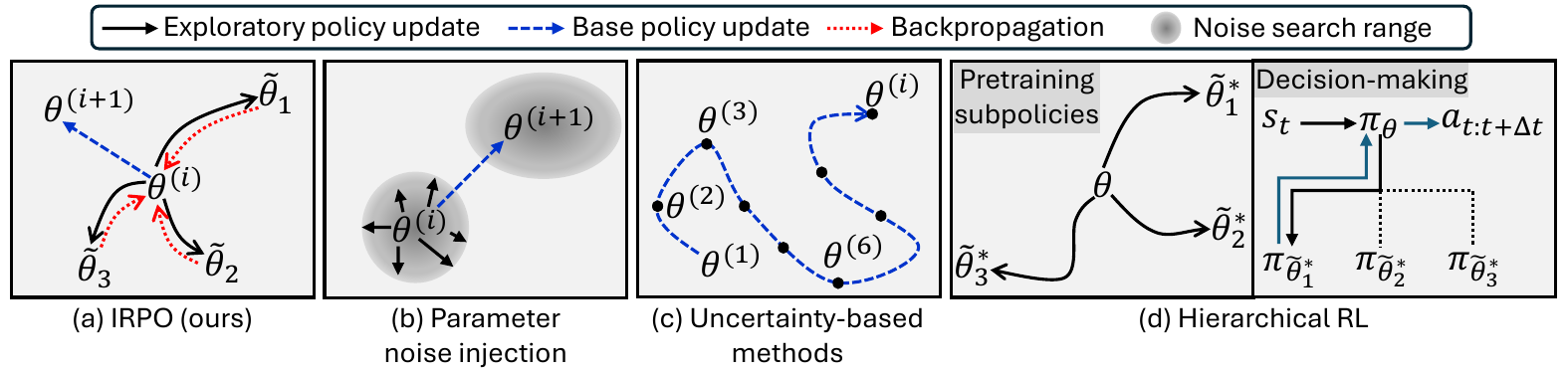}
    \caption{
(a) IRPO optimizes multiple exploratory policies (from the current base policy $\pi_{\theta^{(i)}}$) using intrinsic rewards and uses their gradients to update the base policy via backpropagation.  
(b) Policy noise injection methods generate nearby policies by adding random noise and use their gradients to update the base policy.  
(c) Uncertainty-based methods augment the true (extrinsic) reward with intrinsic rewards based on various forms of uncertainty.
(d) Hierarchical RL pretrains subpolicies (based on intrinsic rewards) and uses them as temporally extended actions.
Unlike the other methods, our approach is less restricted in its search range, directly optimizes for extrinsic rewards, and does not require pretrained subpolicies.
}
    \label{fig:irpo}
\end{figure*}

\section{Introduction}
Reinforcement learning (RL) aims to learn an optimal policy through trial and error. As such, it is important for the agent to thoroughly explore the environment. 
Many exploration strategies have been proposed, such as noise injection in the action space \cite{sutton1999policy, schulman2017proximal, haarnoja2018soft} or in the policy parameter space \cite{fortunato2017noisy, plappert2018parameter} (see \Cref{fig:irpo}b). 
However, these methods do not induce sufficient diversity in collected experiences, limiting their performance in sparse-reward environments. 

An alternative approach adopts intrinsic rewards to provide principled guidance for exploration.
Common strategies include: (1) incentivizing the agent to visit states with high uncertainty, typically characterized by state visitation counts \cite{bellemare2016unifying, tang2017exploration, machado2020count}, policy entropy \cite{schulman2017proximal, haarnoja2018soft}, or uncertainty in statistical models \cite{burda2018exploration, yang2024exploration};
and (2) hierarchical approaches that use pretrained subpolicies---typically based on diffusion-maximizing intrinsic rewards---as temporally extended actions \cite{machado2017laplacian, wu2018laplacian, jinnai2020exploration, gomez2023proper}.

Methods that use uncertainty-based intrinsic rewards typically optimize a policy by maximizing a sum of the true (extrinsic) reward and intrinsic rewards (see \Cref{fig:irpo}c).
These methods enhance exploration, but they complicate credit assignment due to the use of an augmented reward signal \cite{ladosz2022exploration}. 
For example, the agent may fail to recognize task completion when extrinsic rewards are overshadowed by poorly scaled or changing intrinsic rewards.
In addition, count-based state visitation methods do not scale to continuous and high-dimensional state spaces, as learning the density model required for measuring such counts becomes inefficient and noisy \cite{ladosz2022exploration}.
Hierarchical RL avoids poor credit assignment and scales well by directly optimizing extrinsic rewards using subpolicies (see \Cref{fig:irpo}d).
However, temporally extended actions (subpolicies) lead to sub-optimality, as it limits fine-grained decision-making \cite{sutton1999between}, and sample inefficiency, due to pretraining of subpolicies.

We propose an intrinsic reward policy optimization (IRPO) algorithm to address these challenges.
Our key idea is to use a surrogate policy gradient, which we refer to as the IRPO gradient, to directly optimize a policy, replacing the uninformative true gradient in sparse-reward environments.
As shown in \Cref{fig:irpo}a, IRPO constructs the IRPO gradient by backpropagating the extrinsic reward signals of exploratory policies, which are optimized using intrinsic rewards, to a base policy using the chain rule.
In this work, we use diffusion-maximizing intrinsic rewards \cite{gomez2023proper}, though in principle, our approach could use any set of intrinsic rewards.
In spite of the bias in the IRPO gradient, we show that IRPO outperforms relevant baselines in both performance and sample efficiency across discrete and continuous sparse-reward environments.
We also include comprehensive ablations and formally analyze the optimization problem solved by IRPO.

\section{Background}
We model our problem as an Markov decision process (MDP) \cite{puterman2014markov} defined by a tuple $ \mc{M} = ( \mathcal{S}, \mathcal{A}, T, R, \gamma ) $, where $\mc{S}$ is a set of states, $\mc{A}$ is a set of actions, $ T : \mathcal{S} \times \mathcal{A} \times \mathcal{S} \rightarrow \mb{R} $ is the transition probability distribution, $ R : \mathcal{S} \times \mathcal{A} \rightarrow [0, R_{\max}] $ is the reward function bounded by $R_{\max} \in \mathbb{R}$, and $ \gamma \in [0,1) $ is the discount factor. 
At time step $t$, the agent executes an action $a_t \in \mathcal{A}$ given the current state $s_t \in \mathcal{S}$, after which the system transitions to state $s_{t+1} \sim T(\cdot \mid s_t, a_t)$ and the agent receives reward $r_t = R(s_t, a_t)$.
The agent acts based on a stochastic policy $\pi: \mc{S} \times \mc{A} \rightarrow [0, 1]$.

Let $g_t = \sum_{l=0}^{\infty} \gamma^l r_{t+l}$ be the return.
The agent's typical goal in RL is to find a policy $\pi^\star$ that maximizes its expected return from each state $s_t$.
RL algorithms often use the state-value function, $V^\pi(s) = \mathbb{E}_{\pi} \left[ g_t \mid s_t = s \right]$, or the action-value function, $Q^\pi(s,a) = \mathbb{E}_{\pi} \left[ g_t \mid s_t = s, a_t = a \right]$, to express the expected return of the policy $\pi$.


Policy gradient methods \cite{sutton1999policy} optimize a policy $\pi_\theta$ parameterized by $\theta \in \mb{R}^m$ by performing gradient ascent on the performance objective
$J(\theta) =\mathbb{E}_{s_0 \sim d_0} \left[ V^{\pi_\theta}(s_0) \right]$, where $d_0$ is the initial state distribution.
Specifically, we perform the following update:
\begin{equation}\label{eqn:policy_gradient_update}
    \theta^{(i+1)} = \theta^{(i)} + \eta \nabla_{\theta^{(i)}} J(\theta^{(i)}),
\end{equation}
where $i$ is the current iteration, $\eta$ is a learning rate, $\nabla_\theta J(\theta) = \mathbb{E}_{d^{\pi_\theta}} \left[ Q^{\pi_\theta}(s, a) \; \nabla_\theta \log \pi_\theta(a \mid s) \right]$ is the policy gradient, and $d^{\pi_\theta}$ is the discounted state occupancy induced by $\pi_\theta$.
We can also express the policy gradient in a variance-reduced form as follows \cite{schulman2018highdimensionalcontinuouscontrolusing}:
\begin{equation}\label{eqn:policy_gradient_with_advantage}
    \nabla_\theta J(\theta) = \mathbb{E}_{d^{\pi_\theta}} \left[ A^{\pi_\theta}(s, a) \cdot  \nabla_\theta \log \pi_\theta(a \mid s) \right],
\end{equation}
where $A^{\pi_\theta}$ is the advantage function of policy $\pi_\theta$.

Actor-critic algorithms learn a value function to estimate the policy gradient.
In this work, we learn a state-value function $V_\phi$, parameterized by $\phi$, to estimate a TD residual form of the advantage $A^{\pi_\theta}(s_t,a_t) \approx r_t + \gamma V_\phi(s_{t+1}) - V_\phi(s_t)$, where we simplify notation by dropping the policy term $\pi_\theta$.
We update the state-value function by performing the following update:
\begin{equation}\label{eqn:action_value_loss}
    \phi^{(i+1)} = \phi^{(i)} - \eta \nabla_{\phi^{(i)}}\left( \mb{E}_{d^{\pi_\theta}} \left[ (\hat V(s_t) - V_{\phi^{(i)}}(s_t))^2 \right] \right),
\end{equation}
where $\hat V(s_t) \coloneq r_t + \gamma \Big( (1-\lambda) V_{\phi^{(i)}}(s_{t+1}) + \lambda \hat{V}(s_{t+1}) \Big)$ and $ \lambda \in [0, 1]$ a hyperparameter for bootstrapping.

\begin{figure*}[t]
    \centering
    \includegraphics[width=0.9\linewidth]{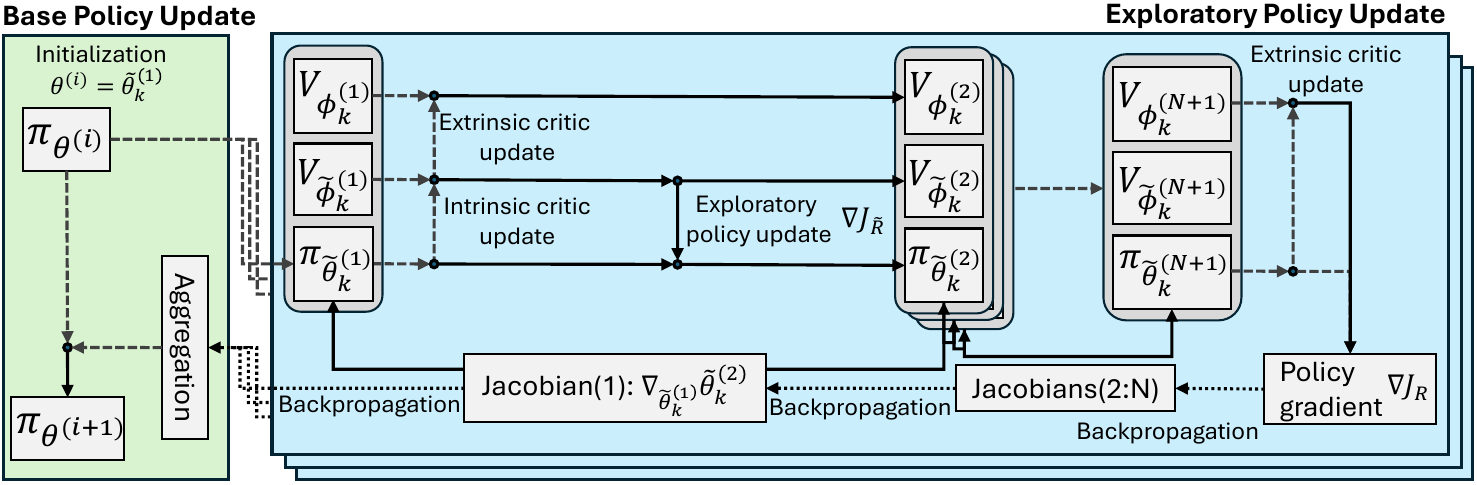}
    \caption{We assume an actor-critic framework with multiple exploratory policies, each having its own intrinsic and extrinsic critics, which are used to update the exploratory and base policy, respectively.}
    \label{fig:architecture}
\end{figure*}

\section{Our Algorithm: Intrinsic Reward Policy Optimization}\label{sec:method}
We propose an algorithm, IRPO, that uses intrinsic rewards to construct a surrogate policy gradient---which we call the IRPO gradient---to directly optimize a policy in sparse-reward environments. 
Assume we are given $ K \geq 1 $ intrinsic reward functions, denoted by $ \tilde{R}_k $ for $ k = 1, 2, \dots, K $.
Each intrinsic reward $\tilde{R}_k$ is then associated with an exploratory policy $\pi_{\tilde{\theta}_k}$, as shown in \Cref{fig:architecture}.
Building on the actor-critic framework, each exploratory policy $\pi_{\tilde{\theta}_k}$ is associated with an intrinsic critic $V_{\tilde{\phi}_k}$, which estimates the value of $\pi_{\tilde{\theta}_k}$ with respect to intrinsic reward ${\tilde{R}_k}$, and an extrinsic critic $V_{\phi_k}$, which estimates the value of $\pi_{\tilde{\theta}_k}$ with respect to the extrinsic reward $R$.
Intrinsic critics are used to optimize the exploratory policies with respect to their corresponding intrinsic reward functions---given appropriately defined intrinsic rewards, these policies should discover strong extrinsic reward signals.
Then, extrinsic critics are used to compute the policy gradients of each exploratory policy with respect to the extrinsic reward, and these gradients are backpropagated to construct the IRPO gradient for updating the base policy $\pi_\theta$ (to maximize the extrinsic reward).

\begin{algorithm}[t]
    \caption{\textbf{IRPO}: Intrinsic reward policy optimization}
    \label{alg:irpo}
    \begin{algorithmic}[1]
        \REQUIRE $K$ intrinsic reward functions $\{ \tilde{R}_k\}_{k=1}^K$
        \item[\textbf{Initialize:}]  Base policy $\pi_\theta$ and $K$ pairs of intrinsic and extrinsic critics $\{ (V_{\tilde{\phi}_k}, V_{\phi_k)}\}_{k=1}^K$
        \FOR{iteration: $i=1,2,\dots$}
            \STATE \texttt{/* Exploratory policy updates */}
            \FOR{intrinsic reward function: $k=1,2,\dots,K$}
            \STATE Initialize exploratory policy $\tilde{\theta}^{(1)}_k = \theta^{(i)}$
                \FOR{exploratory policy update: $j=1,\dots,N$}
                    \STATE Collect rollouts using $\pi_{\tilde{\theta}_k^{(j)}}$
                    \STATE Update intrinsic and extrinsic critics with respect to $\tilde{R}_k$ and $R$, respectively
                    \STATE Update exploratory policy $\tilde{\theta}_k^{(j+1)} = \tilde{\theta}_k^{(j)} + \eta \nabla_{\tilde{\theta}_k^{(j)}} J_{\tilde{R}_k}(\tilde{\theta}^{(j)}_k)$
                    \STATE Compute and store the Jacobian $\nabla_{\tilde{\theta}^{(j)}_k} \tilde{\theta}^{(j+1)}_k$
                \ENDFOR
            \ENDFOR
            \STATE \texttt{/* Base policy update */}
            \FOR{intrinsic reward function: $k=1,2,\dots,K$}
                \STATE Collect rollouts using $\pi_{\tilde{\theta}_k^{(N+1)}}$
                \STATE Compute policy gradient with respect to extrinsic rewards $\nabla_{\tilde{\theta}^{(N+1)}_k} J_R(\tilde{\theta}^{(N+1)}_k)$
            \ENDFOR
            \STATE Compute the IRPO gradient using \Cref{eqn:IRPO_gradient}
            \STATE Update base policy using \Cref{eqn:trust_region_update}
        \ENDFOR
    \end{algorithmic}
\end{algorithm}

We discuss IRPO in more detail below and provide pseudo-code in \Cref{alg:irpo}. The algorithm operates in a bi-level formulation, where given the current base policy, exploratory policies are updated as discussed in \Cref{sec:exploratory_policy_updates}. The base policy is then updated as discussed in \Cref{sec:base_policy_update}.


\subsection{Exploratory Policy Updates}\label{sec:exploratory_policy_updates}
Given the current base policy parameters $\theta$, we first initialize exploratory policies as duplicates, such that $\tilde{\theta}_k^{(1)} = \theta^{(i)}$ for $k = 1, 2, \dots, K$.
Additionally, we initialize the intrinsic and extrinsic critics when the base policy is new ($i=1$) or retain their parameters from the last update otherwise.
We then perform $N > 1$ updates to each exploratory policy using its own intrinsic rewards.

\paragraph{Updating exploratory policies.} 
For the $j$-th update for exploratory policy $\pi_{\tilde{\theta}^{(j)}_k}$, where $ j \in \{1, 2, \dots, N\}$ and $k \in \{1, 2, \dots, K \}$, we first collect rollouts from $\pi_{\tilde{\theta}^{(j)}_k}$.
Each rollout contains experiences of the form $(s_t, a_t, r_t,\tilde{r}_{t,k})$, where $\tilde{r}_{t,k} = \tilde{R}_k(s_t, a_t)$.
We use these experiences to update the intrinsic critic $V_{\tilde{\phi}^{(j)}_k}$ (based on intrinsic rewards) and the extrinsic critic $V_{\phi^{(j)}_k}$ (based on extrinsic rewards) for each exploratory policy using \Cref{eqn:action_value_loss}.
We use the intrinsic critic to estimate the policy gradient using \Cref{eqn:policy_gradient_with_advantage}, which is used to update the exploratory policy using \Cref{eqn:policy_gradient_update}.
We use the final extrinsic critic $V_{\phi^{(N+1)}_k}$ to update the base policy as discussed in \Cref{sec:base_policy_update}.

\paragraph{Storing Jacobians for future base policy updates.}
For each exploratory policy update $j \in \{1, 2, \dots, N\}$, we also calculate and store the Jacobian between the consecutive exploratory policy parameters as follows:
\begin{equation}\label{eqn:the_hessian_between_parameters}
    \begin{aligned}
        \nabla_{\tilde{\theta}^{(j)}_k} \tilde{\theta}^{(j+1)}_k & = \frac{\partial }{\partial \tilde{\theta}^{(j)}_k} \left[ \tilde{\theta}^{(j)}_k + \eta \nabla_{\tilde{\theta}^{(j)}_k} J_{\tilde{R}_k}(\tilde{\theta}^{(j)}_k) \right] \\ 
         & = \mb{I} + \eta \nabla^2_{\tilde{\theta}^{(j)}_k} J_{\tilde{R}_k}(\tilde{\theta}^{(j)}_k),
    \end{aligned}
\end{equation}
where $J_{\tilde{R}_k}(\tilde{\theta}_k)$ denotes the performance objective with respect to intrinsic reward $\tilde{R}_k$ and $\mb{I}$ is the identity matrix.
We use this Jacobian to enable backpropagation of the exploratory policy gradients to update the base policy as discussed in \Cref{sec:base_policy_update}.

\subsection{Base Policy Updates}
\label{sec:base_policy_update}
After performing $N$ updates to each exploratory policy (as described in \Cref{sec:exploratory_policy_updates}), we update the base policy parameters by estimating the IRPO gradient defined in \Cref{eqn:IRPO_gradient} and using it in the trust-region update defined in \Cref{eqn:trust_region_update}.
We denote the final exploratory policy after $N$ updates as $\tilde{\theta}_k = \tilde{\theta}_k^{(N+1)}$ for notational clarity.

\paragraph{IRPO gradient.}
We define the IRPO gradient as
\begin{equation}\label{eqn:IRPO_gradient}
    \nabla J_{\text{IRPO}}(\theta, \{ \tilde{\theta}_k\}_{k\in\{1, 2,\dots,K \} } ) := \sum_{k=1}^K \omega_k \cdot \nabla_{\theta} J_R(\tilde{\theta}_k),
\end{equation}
where the weights $\omega_k$ are defined as follows:
\begin{equation}\label{eqn:definition_of_w}
    \omega_k \coloneq \frac{e^{J_R(\tilde{\theta}_k) / \tau}}{\sum_{k'=1}^K e^{J_R(\tilde{\theta}_{k'}) / \tau}},
\end{equation}
with $\tau \in (0, 1]$ being a temperature parameter.
The backpropagated gradient of each exploratory policy to the base policy is defined as
\begin{equation}\label{eqn:IRPO_backpropagated_gradient}
    \nabla_{\theta} J_R(\tilde{\theta}_k) = \mathbb{E}_{d^ {\pi_{\tilde{\theta}_k}}} \left[ Q^{\pi_{\tilde{\theta}_k}}_R(s,a) \;\nabla_\theta \log \pi_{\tilde{\theta}_k}(a \mid s) \right],
\end{equation}
where $J_R$ and $Q^\pi_R$ denote the performance objective and action-value function of policy $\pi$ with respect to extrinsic rewards $R$, respectively, and the gradient of the logarithm of policy is defined as follows:
\begin{equation}
    \nabla_\theta \log \pi_{\tilde{\theta}_k}(a \mid s) \coloneq (\nabla_{\theta} \tilde{\theta}_k)^\top \cdot \nabla_{\tilde{\theta}_k} \log \pi_{\tilde{\theta}_k}(a \mid s).
\end{equation}


Intuitively, as $\tau \rightarrow 0_+$, updating the base policy with the IRPO gradient guides the base policy to the policy where $N$ exploratory policy updates from that policy will produce an exploratory policy that is near-optimal with respect to the extrinsic reward (see Remark \ref{thm:pseudo_optimality}).


\paragraph{Estimating the IRPO gradient with backpropagation.}
We estimate the IRPO gradient by collecting rollouts from each final exploratory policy and using those rollouts to estimate the policy gradients for each final exploratory policy with respect to the extrinsic rewards $\nabla_{\tilde{\theta}_k} J_R(\tilde{\theta}_k)$, using \Cref{eqn:policy_gradient_with_advantage}.
These gradients are then backpropagated to estimate $\nabla_{\theta} J_R(\tilde{\theta}_k)$ by applying the chain rule using the Jacobian from \Cref{sec:exploratory_policy_updates} as follows:
\begin{equation}\label{eqn:backpropagation_of_extrinsic_policy_gradient}
    \nabla_{\theta} J_R(\tilde{\theta}_k) = \left( \prod_{j=1}^{N} \nabla_{\tilde{\theta}^{(j)}_k} \tilde{\theta}^{(j+1)}_k \right)^\top
        \cdot \nabla_{\tilde{\theta}_k} J_R(\tilde{\theta}_k).
\end{equation}
We discuss the practical computation of this gradient in Appendix \ref{app_subsec:IRPO}.

\paragraph{Trust-region update.} 
We use the following trust-region update \cite{schulman2015trust} to prevent drastic changes to the base policy:
\begin{equation}\label{eqn:trust_region_update}
    \begin{aligned}
        \theta^{(i+1)} & = \argmax_{\theta} \; \nabla J_{\text{IRPO}}^\top \cdot (\theta - \theta^{(i)}), \\
        & \textrm{s.t.} \quad D_{\text{KL}}(\pi_{\theta^{(i)}} \mid \pi_\theta) \leq \delta_{\text{KL}},
    \end{aligned}
\end{equation}
where $\nabla J_{\text{IRPO}} = \nabla J_{\text{IRPO}}(\theta^{(i)}, \{ \tilde{\theta}_k\}_{k\in\{1, 2,\dots,K \} } )$, $ D_{\text{KL}}(\cdot \mid \cdot) $ denotes the Kullback–Leibler (KL) divergence, and $\delta_{\text{KL}} > 0$ is a threshold.
Because we use the IRPO gradient, our update does not inherit the theoretical guarantees of monotonic improvement---however, our empirical results show that this update mechanism produces more stable updates than standard gradient ascent.

\subsection{Formal Analysis of IRPO}\label{sec:theoretical_analysis}
We first show that the true policy gradient vanishes in a sparse-reward setting, motivating the use of the IRPO gradient.
We then formally define the optimization problem that IRPO solves to give insight into how IRPO works and the class of problems for which IRPO can obtain optimality.
Proofs are provided in Appendix \ref{app:theoretical_analysis}.


\begin{corollary}[Vanishing Policy Gradient in Sparse-Reward Settings] \label{cor:vanishment_of_policy_gradient}
     Assume a sparse-reward setting as in Assumption \ref{assump:sparse_rewards}, a bounded policy log-gradient as in Assumption \ref{assump:bound_on_the_policy_log_gradient}, and a discount factor $ \gamma \in [0, 1) $.
    Then, the $\ell_2$-norm of the policy gradient of any stochastic policy $\pi_\theta$ approaches zero as the sparsity of rewards increases as follows:
    \begin{equation}
        \left\| \nabla_\theta J(\theta) \right\|_2 \longrightarrow 0 \quad \text{as } \epsilon \to 0.
    \end{equation}
\end{corollary}

We now define the set of policies that IRPO searches over.
\begin{definition}[Reachable Exploratory Policies]
    \label{def:reachable_policy_set}
    Let $\theta \in \mathbb{R}^m$ be the parameters for a stochastic (base) policy $\pi_\theta$ and $\{\tilde{R}_k \}_{k=1}^K$ be any set of $K$ intrinsic reward functions.
    For each intrinsic reward function, assume we initialize an exploratory policy from the base policy and perform $N$ policy gradient updates, as described in \Cref{sec:exploratory_policy_updates}.
    We define the set of reachable exploratory policy parameters $\tilde{\Theta}_N (\theta)$ for base policy $\pi_\theta$ as follows:
    \begin{equation}\label{eqn:the_definition_of_Theta}
        \begin{aligned}
            & \tilde{\Theta}_N (\theta) \coloneq \\ 
            &\left\{ \tilde{\theta}^{(N+1)}_k \;\Bigg|\; 
        \begin{aligned}
            &\tilde{\theta}^{(N+1)}_k = \tilde{\theta}_k^{(1)} + \eta \sum_{j=1}^{N} \nabla_{\tilde{\theta}_k^{(j)}} J_{\tilde{R}_k}(\tilde{\theta}_k^{(j)}), \\
            & \tilde{\theta}_k^{(1)} = \theta, \quad k=1, \dots, K
        \end{aligned}
        \right\}.
        \end{aligned}
    \end{equation}

\end{definition}

\begin{definition}[All Reachable Exploratory Policies]\label{def:set_of_rechable_policy_set}
    Given Definition \ref{def:reachable_policy_set}, we define the set of reachable exploratory policies $\tilde{\Gamma}_N$ for all possible base policies as follows:
    \begin{equation}
        \tilde{\Gamma}_N \coloneq \bigcup_{\theta \in \mb{R}^m} \tilde{\Theta}_N(\theta).
    \end{equation}
\end{definition}

\begin{remark}[Optimization Problem of IRPO]\label{thm:pseudo_optimality}
Let $\{\tilde{R}_k \}_{k=1}^K$ be a set of intrinsic reward functions, $N$ be the number of exploratory policy updates, and $\theta \in \mathbb{R}^m$ be the parameters for a stochastic (base) policy $\pi_\theta$.
Assume IRPO is implemented such that $\tau$ is annealed to zero over the course of training (i.e., $\tau \rightarrow 0_+$).

Then, $\theta^\dagger$ are the resulting base policy parameters from IRPO, such that
\begin{equation}\label{eqn:IRPO_algorithm_base_policy}
    \begin{aligned}
        \theta^\dagger =&  \argmax_{\theta \in \mathbb{R}^m} \;\max_{k=1,2,\dots,K} J_R (\tilde{\theta}_{k}^{(N+1)}),\\
    \end{aligned}
\end{equation}
where $\tilde{\theta}^{(N+1)}_k$ are the exploratory policy parameters for intrinsic reward $\tilde{R}_k$ updated from base policy parameters $\theta$, such that
\begin{equation}\label{eqn:definition_of_theta_tilde_k}
    \tilde{\theta}^{(N+1)}_k = \theta + \eta \sum_{j=1}^{N} \nabla_{\tilde{\theta}_k^{(j)}} J_{\tilde{R}_k}(\tilde{\theta}_k^{(j)}), \; \forall k = 1, 2, \dots, K.
\end{equation}

Let $q$ be the index of the exploratory policy that maximizes the performance objective from $\theta^\dagger$, such that
\begin{equation}
    q = \argmax_{k=1,2,\dots, K} J_R (\tilde{\theta}_{k, \dagger}^{(N+1)}),
\end{equation}
where 
\begin{equation}
    \tilde{\theta}^{(N+1)}_{k, \dagger} = \theta^\dagger + \eta \sum_{j=1}^{N} \nabla_{\tilde{\theta}_k^{(j)}} J_{\tilde{R}_k}(\tilde{\theta}_k^{(j)}), \; \forall k = 1, 2, \dots, K.
\end{equation}

Then, the output policy parameters of IRPO, $\tilde{\theta}^{(N+1)}$, satisfy the following:
\begin{equation}\label{eqn:pseudo_optimal_policy}
    \tilde{\theta}^{(N+1)} = \argmax_{\tilde{\theta} \in \tilde{\Gamma}_N} J_R(\tilde{\theta}),
\end{equation}
where $\tilde{\Gamma}_N$ is defined in Definition \ref{def:set_of_rechable_policy_set}.
\end{remark}



If we assume that the parameters for an optimal policy $\theta^\star$ are within the set of all reachable exploratory policies, that is $\theta^\star \in \tilde{\Gamma}_N$, then IRPO can obtain optimality. 
This assumption may hold, for example, when we have a large set of intrinsic rewards and a small number of exploratory updates $N$.
However, using a large number of intrinsic reward functions increases sample complexity, and decreasing $N$ reduces exploration, which could hurt performance in complex environments.
Despite this apparent contradiction, we empirically show in \Cref{sec:experiments} that IRPO can obtain optimality in our considered discrete environments and near-optimality in our considered continuous environments (we use $N=5$ in our experiments).

\begin{figure}
    \centering
    \includegraphics[width=1.0\linewidth]{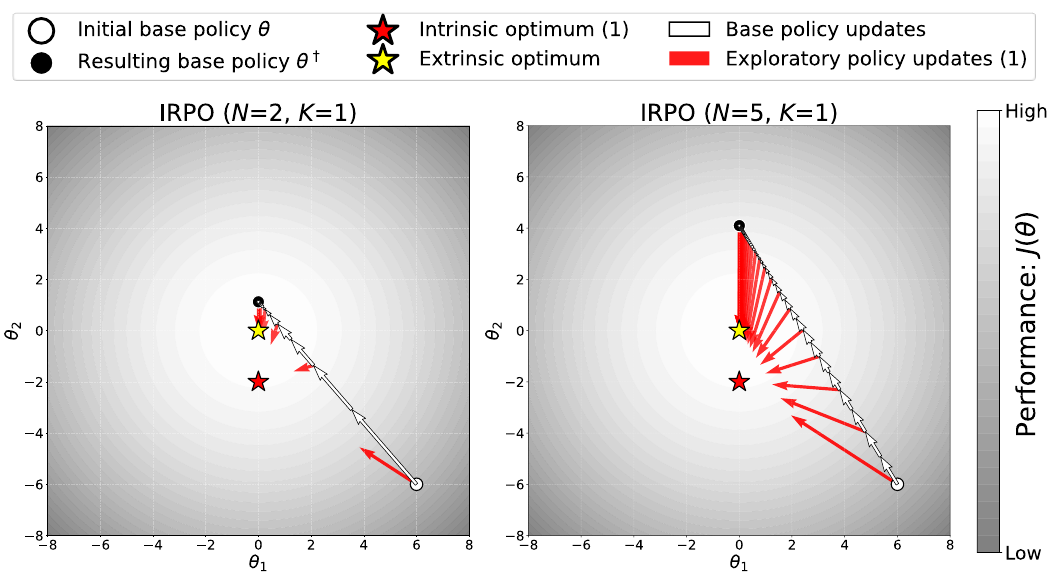}
    \caption{We show the update mechanism of IRPO with a single intrinsic reward (i.e., $K = 1$). We plot the results for different numbers of exploratory policy updates, $N=2$ and $N=5$, respectively, with extrinsic performance plotted on a contour.}
    \label{fig:analysis1}
\end{figure}

\subsection{Empirical Analysis of IRPO}
We provide a visual illustration of IRPO's update mechanism in \Cref{fig:analysis1}. This example assumes a two-dimensional parameter space $\theta = [\theta_1, \theta_2]^\top \in \mathbb{R}^2$, with the following (strictly concave) extrinsic and intrinsic performance objectives:
\begin{equation}\label{eqn:ext_and_int_objectives}
    J(\theta) \coloneq -\| \theta \|^2_2, \quad 
    \tilde{J}_1(\theta) \coloneq -\left\| \theta - \begin{bmatrix} 0 \\ -2 \end{bmatrix} \right\|^2_2.
\end{equation}
As expected, the exploratory policy updates (red arrows) are directed towards the intrinsic performance objective (red star).
The base policy updates---driven by the IRPO gradient---guide the base policy toward a region where $N$ exploratory updates (with respect to the intrinsic objective) will land at the extrinsic optimum.
We also see that increasing the number of exploratory policy updates $N$ increases exploration (indicated by the length of red arrows). 
We provide similar analysis for multiple intrinsic rewards with varying temperature parameters $\tau$ in Appendix \ref{app:empirical_analysis}.

\section{Experiments}\label{sec:experiments}
We empirically compare IRPO to relevant baselines in discrete and continuous sparse-reward environments used in the literature \citet{machado2017laplacian, wu2018laplacian, wang2021towards, yang2024exploration}.
Our environments are visualized in \Cref{fig:env} and detailed in Appendix \ref{app:environmental_parameters}.

\begin{figure}[t]
    \centering
    \includegraphics[width=0.9\linewidth]{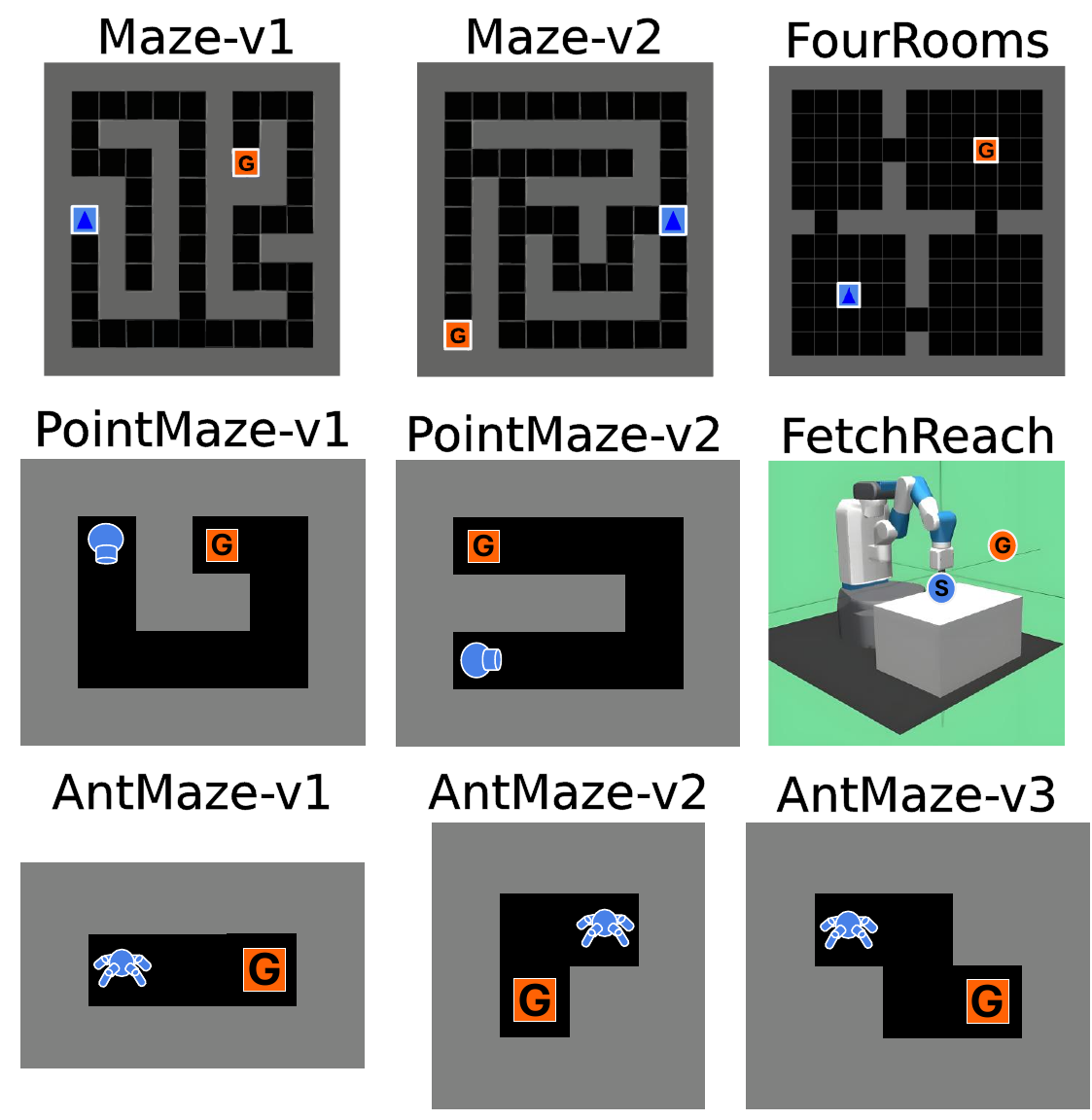}
    \caption{The environments used in our experiments. We consider three discrete environments (top row) and six continuous ones (middle and bottom rows).}
    \label{fig:env}
\end{figure}

\begin{figure*}[t]
    \centering
    \includegraphics[width=0.99\linewidth]{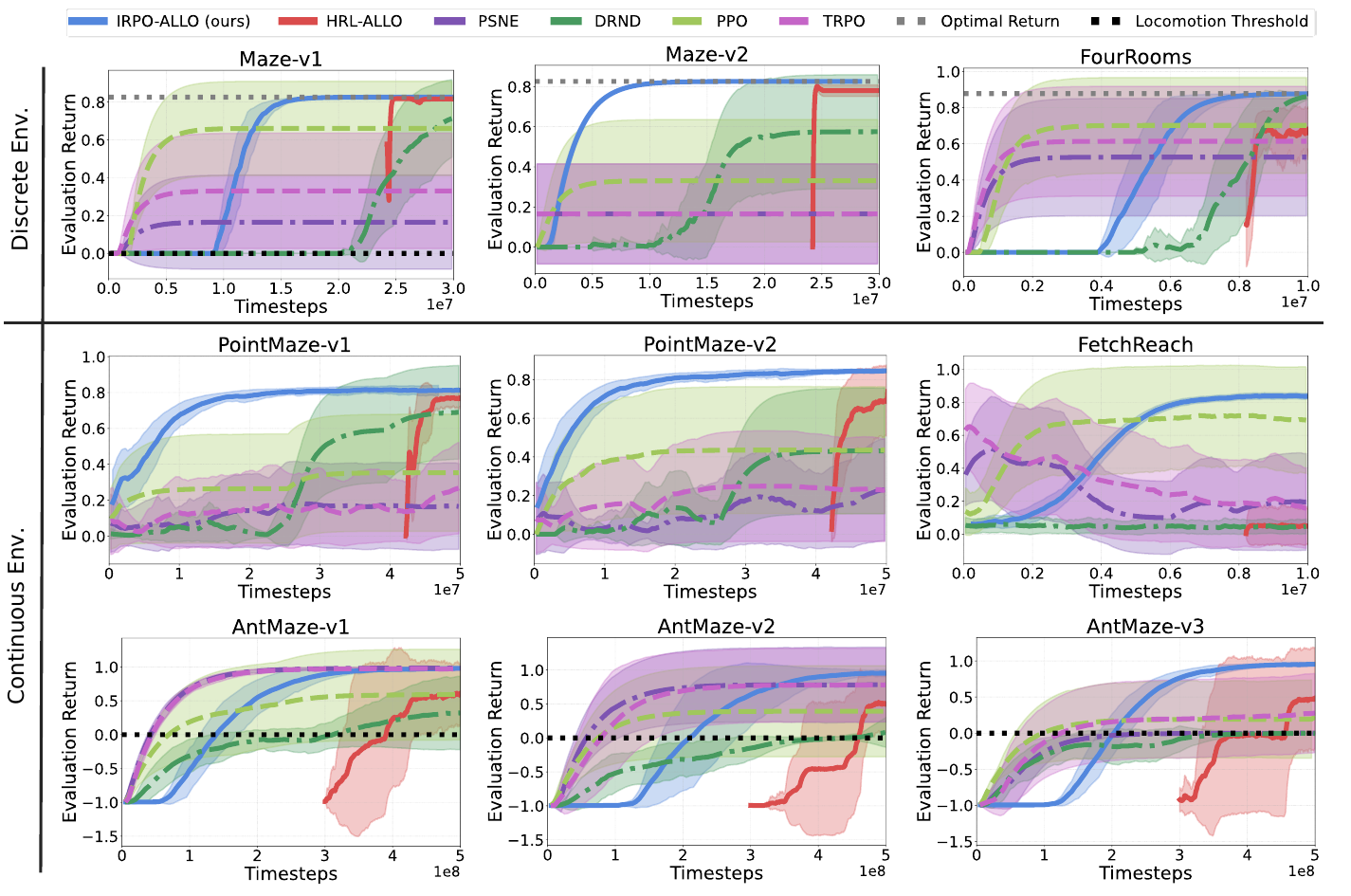}
    \caption{Learning curves for environments in \Cref{fig:env}, with mean and 95\% confidence intervals over 10 random seeds. Learning curves for IRPO-ALLO (ours) and HRL-ALLO include samples needed to derive intrinsic rewards (both) and to pretrain subpolicies (HRL-ALLO only). The locomotion threshold in AntMaze environments refers to the stage where the agent starts learning to walk.
    }
    \label{fig:data}
\end{figure*}

Our baselines consider canonical RL algorithms designed for effective exploration: 
hierarchical RL (HRL) \cite{sutton1999between}, distributional random network distillation (DRND) \cite{yang2024exploration}, parameter space noise for exploration (PSNE) \cite{plappert2018parameter}, proximal policy optimization (PPO) \cite{schulman2017proximal}, and trust region policy optimization (TRPO) \cite{schulman2015trust}.
We provide algorithm details in Appendix \ref{app:network_parameters}.


For IRPO and HRL, we use one set of intrinsic rewards derived from the augmented Lagrangian Laplacian objective (ALLO) \cite{gomez2023proper}, which maximizes agent diffusion within an environment, across all experiments.
We visualize these intrinsic rewards in Appendix \ref{app:intrinsic_reward_maps}.

\subsection{Main Results}\label{subsec:performance_analysis}

\Cref{fig:data} shows learning curves for each environment. 
To ensure a fair comparison, all samples used to generate ALLO-based intrinsic rewards, pretrain subpolicies, and perform exploratory policy updates are included in the learning curves of relevant algorithms.


We see that IRPO consistently attains the highest converged performance and maintains narrow confidence intervals in all environments except \emph{FourRooms} and \emph{AntMaze-v1}, where one or two baselines achieve similar results.
Such strong performance indicates that IRPO provides sufficient exploration (via exploratory policy updates) to find extrinsic rewards and assign credit well (via base policy updates).
Similarly, HRL-ALLO demonstrates effective exploration and credit assignment, outperforming other baselines in more than half of the environments (\emph{Maze-v1, Maze-v2, PointMaze-v1, PointMaze-v2, and AntMaze-v3}), although it consistently underperforms relative to IRPO.
Notably, in \emph{FetchReach}, HRL-ALLO completely fails to learn, while IRPO achieves strong performance using the same intrinsic rewards.
This result aligns with a well-known limitation of hierarchical RL, where the use of temporally extended actions can restrict fine-grained decision-making \cite{sutton1999between}. 
PSNE and TRPO---which directly optimize extrinsic rewards---struggle in most settings except \emph{AntMaze-v1 and AntMaze-v2}, which we attribute to insufficient exploration.
PPO performs moderately well, but is still significantly outperformed by IRPO.
DRND---which builds upon PPO---yields negligible improvements to PPO in all discrete environments and \emph{PointMaze-v1}, and completely fails to learn in \emph{FetchReach}. We attribute this failure to the difficulty of credit assignment when optimizing the combined sum of rewards.

\begin{figure*}[t!]
    \centering
    \includegraphics[width=0.99\linewidth]{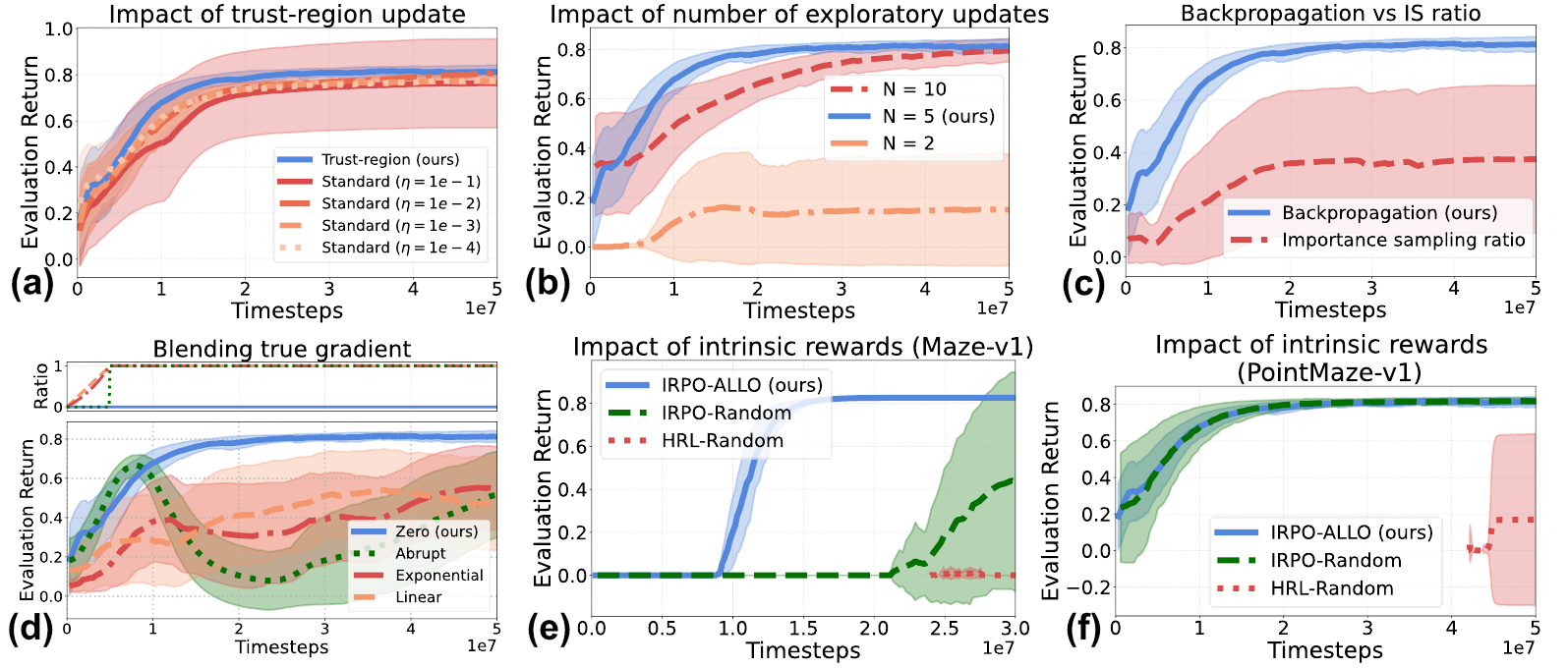}
    \caption{
    Learning curves for our ablation studies, with mean and 95\% confidence intervals over 10 random seeds.}
    \label{fig:ablation}
\end{figure*}

Regarding sample complexity, we see that IRPO achieves lower sample complexity than HRL-ALLO across all environments, due to the fact that HRL-ALLO requires pretraining a subpolicy for each intrinsic reward. 
IRPO also shows lower sample complexity than DRND across all environments, likely due to DRND struggling to optimize its continuously changing intrinsic rewards.
PSNE, PPO, and TRPO show lower sample complexity than IRPO in several environments, including \emph{Maze-v1}, \emph{FourRooms}, \emph{FetchReach}, and \emph{AntMaze} environments.
This is because these baselines do not rely on complex exploration strategies that require additional samples (e.g., pretraining subpolicies or exploratory updates).
However, while PSNE, PPO, and TRPO show lower sample complexity than IRPO in \emph{Maze-v1}, \emph{FourRooms}, and \emph{FetchReach}, their converged performance is significantly lower than IRPO's.
Furthermore, as we increase the maze complexity in the \emph{AntMaze} environments (from v1 to v3), the converged performance and variance of these baselines also degrades, whereas IRPO maintains strong performance.
We also compare the overall wall-clock time of IRPO against PPO in Appendix \ref{app:computational_time}.



%
Finally, although the search space of IRPO is limited (see \Cref{eqn:pseudo_optimal_policy} in Remark \ref{thm:pseudo_optimality}), our experiments demonstrate that it achieves optimality in discrete environments and near-optimality in continuous environments. 
This result suggests that our hyperparameters (i.e., $N=5$ exploratory updates and $K$ intrinsic rewards), combined with the expressivity of neural networks, satisfy the conditions necessary to recover true optimality (see discussion following Remark \ref{thm:pseudo_optimality}).


\subsection{Ablation Study}
\Cref{fig:ablation} shows learning curves for our ablation study.
We compare our trust-region update against standard gradient updates with various learning rates in \Cref{fig:ablation}a.
We see that the trust-region update significantly reduces observed confidence intervals relative to standard updates, while also achieving slightly faster and higher converged performance.



\Cref{fig:ablation}b shows the impact of the number of exploratory policy updates $N$ (which controls exploration) on performance.
We see that $N = 2$ results in poor converged performance, likely due to insufficient exploration, while $N = 10$ results in poor sample efficiency with no improvement in converged performance. 
This result suggests IRPO requires tuning to find a value for $N$ that effectively balances exploration and sample complexity.

\Cref{fig:ablation}c compares the use of our backpropagated IRPO gradient with a naive importance sampling (IS) alternative.
In principle, IS can be used to provide an unbiased update for the base policy based on exploratory policy samples---our IRPO gradient is instead biased.
We test this idea by using the following IS-based policy gradient:
\begin{equation}\label{eqn:unbiased_IRPO_gradient}
    \begin{aligned}
        \nabla_{\theta} J_R(\theta) \approx \mathbb{E}_{d^{\pi_{\tilde{\theta}_k}}} \biggl[ \rho(s,a) \; Q^{\pi_{\tilde{\theta}_k}}_R(s, a) \; \nabla_\theta \log \pi_\theta(a\mid s) \biggr],
    \end{aligned}
\end{equation}
where $\rho(s,a) \coloneq \pi_\theta(a \mid s)/\pi_{\tilde{\theta}_k}(a \mid s)$ denotes the IS ratio.
We use this gradient to replace \Cref{eqn:IRPO_backpropagated_gradient} in \Cref{eqn:IRPO_gradient} for updating a base policy and use the base policy for evaluation.
The approximation sign arises from neglecting the discounted state occupancy correction term (i.e., $d^{\pi_\theta}(s)/d^{\pi_{\tilde{\theta}_k}}(s)$)---we neglect this term because estimating it requires training additional models---and from using the critic of the exploratory policy to mitigate vanishing gradients as highlighted in Corollary \ref{cor:vanishment_of_policy_gradient}.
We empirically saw that using the critic of the exploratory policy resulted in better performance than using the critic of the base policy.
We see that using the IS-based policy gradient results in poor converged performance and high confidence intervals---this result is not surprising, given that prior work has shown IS corrections between substantially different policies can introduce high variance into policy gradient estimates \cite{sutton1998reinforcement}.

While IRPO does not theoretically produce the optimal policy (see Remark \ref{thm:pseudo_optimality}), one way to recover optimality is to transition to using the true policy gradient after the base policy begins to receive positive rewards.
\Cref{fig:ablation}d shows results for abrupt, exponential, and linear blending of the IRPO gradient with the true gradient during training---we see that all transitions reduce converged performance and increase training instability and variance.
One possible reason for this result is that the two gradients do not have the same optimization landscapes due to the bias of the IRPO gradient.

Finally, we evaluated the sensitivity of IRPO to the quality of the intrinsic rewards used by considering random intrinsic rewards.
We generated random rewards by using a randomly initialized neural network to map states to $K$-dimensional vectors---we visualize the resulting rewards in Appendix \ref{app:intrinsic_reward_maps}.
\Cref{fig:ablation}e shows that IRPO's performance decreases with random intrinsic rewards in \emph{Maze-v1}, though it still significantly outperforms the HRL baseline using those random rewards.
\Cref{fig:ablation}f shows that IRPO achieves similar performance with random intrinsic rewards in \emph{PointMaze-v1} (though with more variance), and still significantly outperforms the comparable HRL baseline, suggesting it is somewhat robust to the intrinsic rewards used.



\section{Related Works}\label{sec:related_works}
Parameter noise injection \cite{fortunato2017noisy,plappert2018parameter} has been used to find a more consistent and structured way to collect diverse experiences, in contrast to action noise injection \cite{ladosz2022exploration}.
However, this approach often struggles in sparse-reward environments and scales poorly with neural network size \cite{plappert2018parameter}.
Furthermore, the additional randomness introduced by noise injection in policy gradient estimation could introduce extra variance to the gradient.

Several works use uncertainty-based intrinsic rewards to drive exploration by optimizing a policy with the sum of intrinsic rewards and an extrinsic reward.
One approach uses the uncertainty of the neural network \cite{ladosz2022exploration}. 
For example, \citet{burda2018exploration, yang2024exploration} introduce two neural networks that take the state as input---one fixed as a target network and the other learned to predict outputs of the target network using states from the agent’s rollouts.
\cite{burda2018exploration, yang2024exploration} use prediction error as an intrinsic reward. 
Since this error tends to be higher for out-of-distribution states, it effectively encourages the agent to explore unvisited areas.
However, due to the non-stationary intrinsic rewards and the challenge of appropriately scaling them, the resulting policy often suffers from poor credit assignment, as evidenced by our experimental results in \Cref{fig:data}.
Additionally, \cite{bellemare2016unifying} introduce density model to approximate the state visitation counts of the agent in continuous environments and uses such counts to drive the agent to visit less-visited states.
However, this also comes at the cost of increased computational complexity and often results in poor performance due to inaccurate density approximation.

Hierarchical RL with Laplacian-based intrinsic rewards has been widely used to solve sparse-reward environments. 
These rewards are typically derived by decomposing a diffusive representation of the transition dynamics, where the top eigenvectors guide the agent toward distant regions from the initial state. 
Graph Laplacians \cite{wang2021towards, wu2018laplacian, machado2017laplacian, gomez2023proper} and successor representations \cite{dayan1993improving, machado2017eigenoption, ramesh2019successor} are prominent techniques used to construct such representations. 
However, hierarchical RL inherently suffers from sub-optimality and sample complexity, as demonstrated in \Cref{fig:data}.


\section{Conclusion}
We propose an algorithm, IRPO, that uses intrinsic rewards to directly optimize a policy for extrinsic rewards in sparse-reward environments. 
Our key idea is to use a new surrogate policy gradient, called the IRPO gradient, that is computed by optimizing exploratory policies using intrinsic rewards from a base policy and backpropagating gradients of those exploratory policies to the base policy using the chain rule.
We empirically show that IRPO outperforms relevant baselines in a wide range of sparse-reward environments, including discrete and continuous ones.
We also formally analyze the optimization problem that IRPO solves.
Future directions include extending this framework to non-navigation tasks, such as dexterous manipulation, using alternative intrinsic rewards (we use Laplacian-based intrinsic rewards in this work), and improving sample efficiency.

\section*{Impact Statement}
This paper presents work whose goal is to advance the field of reinforcement learning. 
The proposed algorithm improves performance in sparse-reward environments and its potential societal consequences are consistent with those of the broader field of reinforcement learning. 
We do not believe that there are any specific negative ethical consequences or societal impacts that must be highlighted here.




\bibliography{example_paper}
\bibliographystyle{icml2026}

\newpage
\appendix
\onecolumn


\section{Proof of Corollary \ref{cor:vanishment_of_policy_gradient}}\label{app:theoretical_analysis}

%


\begin{assumption}[Sparse Reward]\label{assump:sparse_rewards}
The environment has a bounded reward function $R: \mc{S} \times \mc{A} \rightarrow [0, R_{\max}]$ that is sparse, such that the probability of having a positive reward under any stochastic policy $\pi$ is bounded by a constant $\epsilon \in (0, 1)$ as follows:
\begin{equation}\label{eqn:positive_reward_probability}
    0 < \mathbb{P}_{d^\pi}\left(R(s, a) > 0 \right) \leq \epsilon.
\end{equation}
\end{assumption}

\begin{assumption}[Bounded Gradient of the Logarithm of Policy]\label{assump:bound_on_the_policy_log_gradient}
The $\ell_2$-norm of the gradient of the logarithm of a stochastic policy $\pi_\theta$ is uniformly bounded by a constant $\kappa \geq 0$ as follows:
\begin{equation}
    \left\| \nabla_\theta \log \pi_\theta(a \mid s) \right\|_2 \leq \kappa, \quad \forall \theta\in\mb{R}^m,\; s \in \mathcal{S}, \; a \in \mathcal{A}.
\end{equation}
\end{assumption}

Under the sparse-reward setting in Assumption \ref{assump:sparse_rewards}, we first derive a bound on the expected action-value function under the discounted state occupancy $d^\pi$ of any stochastic policy $\pi$.
\begin{lemma}
    [Bounded Expected Action-Value Function in Sparse-Reward Settings] Assume a sparse-reward setting as in  Assumption \ref{assump:sparse_rewards} and a discount factor $ \gamma \in [0, 1) $.
    Then, the expected action-value function under the discounted state occupancy  $d^\pi$ of any stochastic policy $\pi$ is bounded by
    \begin{equation}
        \mb{E}_{d^\pi} \left[ Q^{\pi}(s,a)\right] \leq \frac{\epsilon R_{\max}}{1 - \gamma}.
    \end{equation}
    \label{lemma:bound_on_action_value_function}
\end{lemma}
\begin{proof}
    Recall that the expected action-value function under the discounted state occupancy $d^\pi$ is given as follows:
    \begin{equation}
        \mb{E}_{d^\pi}\left[Q^{\pi}(s,a)\right] = \mathbb{E}_{d^\pi} \left[ \sum_{l=0}^\infty \gamma^l R(s_{t+l}, a_{t+l}) \right].
    \end{equation}

    We upper bound this expectation as follows:
    \begin{equation}
        \begin{aligned}
            \mathbb{E}_{d^\pi} \left[ \sum_{l=0}^\infty \gamma^l R(s_{t+l}, a_{t+l})  \right] \overset{(1)}{=} \sum_{l=0}^\infty \gamma^l \mathbb{E}_{d^\pi} \left[ R(s_{t+l}, a_{t+l})  \right] \overset{(2)}{\leq} \epsilon R_{\max}\sum_{l=0}^\infty \gamma^l  \overset{(3)}{=} \frac{\epsilon R_{\max}}{1 - \gamma},
        \end{aligned}
    \end{equation}
    where (1) uses the linearity of expectation, (2) uses the following definition:
    \begin{equation}
        \begin{aligned}
            \mathbb{E}_{d^\pi}[R(s,a)] &= \mathbb{E}_{d^\pi}[R \mid R = 0] \cdot \mathbb{P}_{d^\pi}(R = 0) + \mathbb{E}_{d^\pi}[R \mid R > 0] \cdot \mathbb{P}_{d^\pi}(R > 0) \\
        &\leq R_{\max} \cdot \mathbb{P}_{d^\pi}(R(s,a) > 0) \leq \epsilon R_{\max}, \quad (\text{Assumption \ref{assump:sparse_rewards}})
        \end{aligned}
    \end{equation}
    and (3) uses the formula for a geometric series.
    
\end{proof}

Given a bounded expected action-value function and Assumption \ref{assump:bound_on_the_policy_log_gradient}, we now show that the $\ell_2$-norm of the policy gradient for any stochastic policy $\pi_\theta$ is bounded.
\begin{lemma}
    [Bounded Policy Gradient in Sparse-Reward Settings] Assume a sparse-reward setting as in Assumption \ref{assump:sparse_rewards}, a bounded policy log-gradient as shown in Assumption \ref{assump:bound_on_the_policy_log_gradient}, and a discount factor $ \gamma \in [0, 1) $.
    Then, the $\ell_2$-norm of the gradient of any stochastic policy $\pi_\theta$ is uniformly bounded by
    \begin{equation}
        \|\nabla_\theta J(\theta)\|_2 \leq \frac{\epsilon R_{\max} \kappa}{1 - \gamma}, \quad \forall \theta\in\mb{R}^m.
    \end{equation}
    \label{lemma:bound_on_the_policy_gradient}
\end{lemma}

\begin{proof}
    Recall the definition of the policy gradient from \Cref{eqn:policy_gradient_update} and consider its $\ell_2$-norm
    \begin{equation}\label{eqn:l2_norm_of_policy_gradient}
        \left\| \nabla_\theta J(\theta) \right\|_2 = \left\| \mathbb{E}_{d^{\pi_\theta}} \left[ Q^{\pi_\theta}(s, a) \nabla_\theta \log \pi_\theta(a \mid s) \right] \right\|_2.
    \end{equation}
    
    We apply Jensen's inequality, exploiting the convexity of the $\ell_2$-norm.
    Specifically, for a convex function $\varphi(x)$, the inequality $\varphi(\mathbb{E}[X]) \leq \mathbb{E}[\varphi(X)]$ holds. Applying this to the policy gradient expression in \Cref{eqn:l2_norm_of_policy_gradient} above yields the following upper bound:
    \begin{equation}
        \begin{aligned}
        \left\| \nabla_\theta J(\theta) \right\|_2
        & = \left\| \mathbb{E}_{d^{\pi_\theta}} \left[ Q^{\pi_\theta}(s,a) \nabla_\theta \log \pi_\theta(a \mid s) \right] \right\|_2, \\
        &\overset{(1)}{\leq} \mathbb{E}_{d^{\pi_\theta}} \left[ \left\| Q^{\pi_\theta}(s,a) \nabla_\theta \log \pi_\theta(a \mid s) \right\|_2 \right],  \\
        &\overset{(2)}{\leq} \mathbb{E}_{d^{\pi_\theta}} \left[ \left| Q^{\pi_\theta}(s,a) \right| \cdot \left\| \nabla_\theta \log \pi_\theta(a \mid s) \right\|_2 \right], \\
        &\overset{(3)}{\leq} \kappa \cdot \mathbb{E}_{d^{\pi_\theta}} \left[ Q^{\pi_\theta}(s,a) \right],  \\
        &\overset{(4)}{\leq}  \frac{\epsilon R_{\max} \kappa}{1 - \gamma},
        \end{aligned}
    \end{equation}
    where (1) uses Jensen's inequality, (2) uses the absolute homogeneity of the norm, (3) uses Assumption \ref{assump:bound_on_the_policy_log_gradient} with $Q(s,a) \geq 0$, and (4) uses Lemma \ref{lemma:bound_on_action_value_function}.
\end{proof}

We now introduce a direct corollary of Lemma \ref{lemma:bound_on_the_policy_gradient} under the sparse-reward assumption (Assumption \ref{assump:sparse_rewards}).

\noindent\textbf{Corollary \ref{cor:vanishment_of_policy_gradient}}
(Vanishing Policy Gradient in Sparse-Reward Settings). 
\emph{
Assume a sparse-reward setting as in Assumption \ref{assump:sparse_rewards}, a bounded policy log-gradient as shown in Assumption \ref{assump:bound_on_the_policy_log_gradient}, and a discount factor $ \gamma \in [0, 1) $.
    Then, the $\ell_2$-norm of the policy gradient of any stochastic policy $\pi_\theta$ approaches zero as the sparsity of rewards increases as follows:
    \begin{equation}
        \left\| \nabla_\theta J(\theta) \right\|_2 \longrightarrow 0 \quad \text{as } \epsilon \to 0.
    \end{equation}
}
\begin{proof}
This directly follows from Lemma \ref{lemma:bound_on_the_policy_gradient}, where we see that as rewards get sparser (i.e., $\epsilon \rightarrow 0$), the policy gradient approaches zero.
\end{proof}

\begin{figure*}[t]
    \centering
    \includegraphics[width=0.95\linewidth]{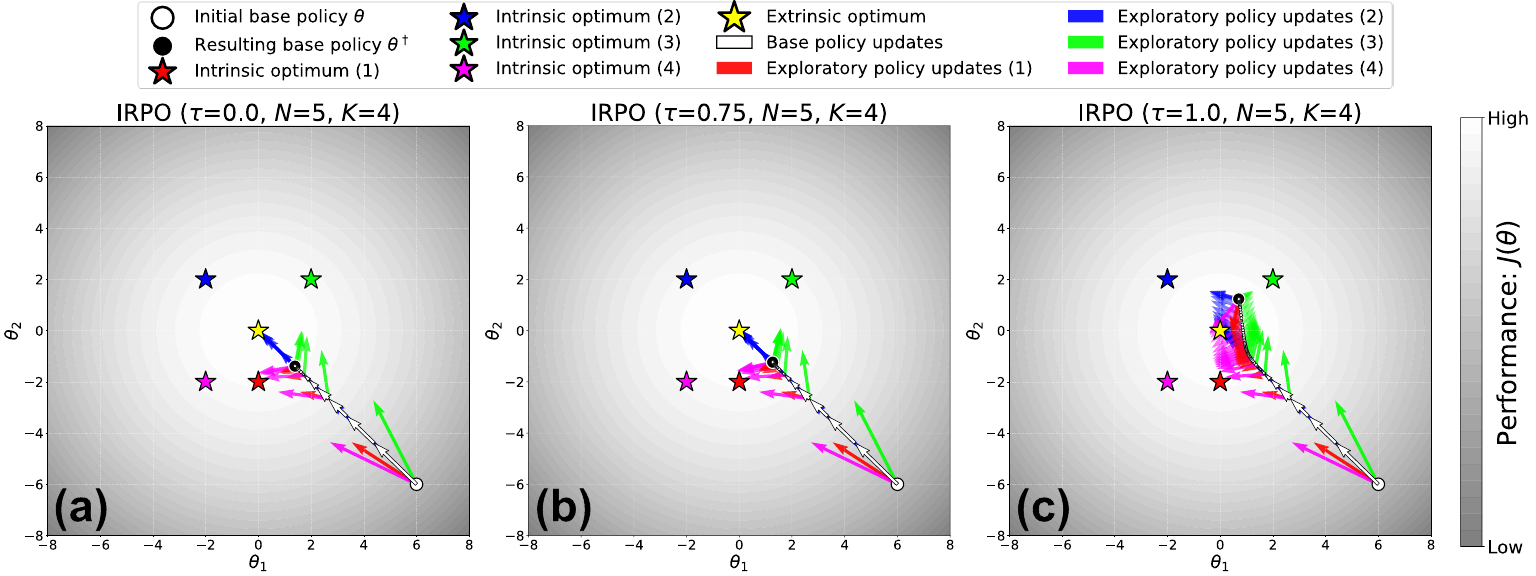}
    \caption{We visualize the update mechanism of IRPO for $K$ intrinsic rewards and $N$ exploratory policy updates, with a varying temperature parameter $\tau$.}
    \label{fig:analysis2}
\end{figure*}

\section{Empirical Analysis}\label{app:empirical_analysis}
\Cref{fig:analysis2} provides a visual illustration of IRPO's update mechanism with multiple intrinsic rewards and the impact of the temperature parameter $\tau$ on the gradient aggregation mechanism introduced in \Cref{eqn:IRPO_gradient}. 
We introduce three additional intrinsic performance objectives (see others in \Cref{eqn:ext_and_int_objectives})  to facilitate the analysis as follows:
\begin{equation}
    \tilde{J}_2(\theta) =  -\left\| \theta - \begin{bmatrix} -2 \\ 2 \end{bmatrix} \right\|^2_2,\quad 
    \tilde{J}_3(\theta) = -\left\| \theta - \begin{bmatrix} 2 \\ 2 \end{bmatrix} \right\|^2_2,\quad 
    \tilde{J}_4(\theta) =  -\left\| \theta - \begin{bmatrix} -2 \\ -2 \end{bmatrix} \right\|^2_2.
\end{equation}
We report the results for $\tau \in \{0.0, 0.75, 1.0\}$.

From \Cref{fig:analysis2}a and \Cref{fig:analysis2}b, it is evident that the base policy converges to a location where $N$ exploratory updates maximizing $\tilde{J}_2$---the intrinsic objective yielding the highest extrinsic performance gain---reaches the extrinsic optimum. 
Although increasing $\tau$ promotes the aggregation of gradients from multiple objectives, we observe that the behavior at $\tau=0.75$ remains similar to the greedy baseline ($\tau=0.0$). 
However, for $\tau=1.0$ (\Cref{fig:analysis2}c), the base policy shifts to a region where most of the exploratory policies (driven by $\tilde{J}_1, \tilde{J}_2, \text{and } \tilde{J}_4$) converge near the extrinsic optimum, thereby maximizing the aggregate performance of the exploratory policies.

\section{Environment Details} \label{app:environmental_parameters}
In each environment visualized in \Cref{fig:env}, an agent is initialized at a fixed starting location (denoted as blue), and the goal is placed at a distant location (denoted as orange). 
For discrete environments, a state observed by the agent consists of the Cartesian coordinates of the agent and the goal, and we use the following action set $ \mathcal{A} = \{ \text{left}, \text{up}, \text{right}, \text{down} \} $.
For continuous environments, the state includes the Cartesian coordinates of the agent and the goal, and the agent's kinematic information, such as velocity. 
Additionally, we adopt the default action set provided by \href{https://robotics.farama.org/index.html}{\texttt{Gymnasium-Robotics}} \cite{gymnasium_robotics2023github} for each environment.
For all environments, the agent receives a $+1$ reward for reaching the goal. 
In \emph{AntMaze} environments, we also give a $-1$ penalty and terminate the episode if the agent jumps aggressively or turns over, ensuring it first learns to walk effectively.
Specifically, we require the vertical distance from the center of the torso to the ground to be between 0.35 and 1.00.
Additionally, we noticed that the default success threshold (distance $< 0.35$) required the agent to perform unnatural jumps to reach the goal, which made the \emph{AntMaze} environments dramatically difficult. 
To address this, we increased the success threshold from 0.35 to 0.5.
Environment parameters---including discount factor, horizon, and state and action dimensions---for each environment are summarized in \Cref{tab:environmental_parameters}.
\begin{table}[t!]
    \centering
    \small
    \caption{Environment parameters used in our experiments.}
    \label{tab:environmental_parameters}
    \begin{tabular}{l | cccccc}
        \toprule\toprule
        \textbf{Environment} & $\gamma$ & Horizon & State Dimension & Action Dimension & Number of Subpolicies\\
        \midrule
        Maze-v1        & 0.99  & 300 & 4  & 4 & 6 \\
        Maze-v2        & 0.99  & 300 & 4  & 4 & 6 \\
        FourRooms      & 0.99  & 100 & 4  & 4 & 4 \\
        PointMaze-v1   & 0.999  & 500 & 8  & 2 &  4\\
        PointMaze-v2   & 0.999 & 500 & 8  & 2 & 4 \\
        FetchReach     & 0.99  & 50 & 16 & 4 & 4 \\
        AntMaze-v1   & 0.9999  & 1000 & 27 & 8 & 3 \\
        AntMaze-v2   & 0.9999 & 1000 & 27 & 8 & 3 \\
        AntMaze-v3     & 0.9999  & 1500 & 27 & 8 & 3 \\
        \bottomrule\bottomrule
    \end{tabular}
\end{table}

\section{Algorithm Details} \label{app:network_parameters}
We present algorithm hyperparameters used for our experiments in \Cref{tab:algorithmic_parameters}, which we found them by a hyperparameter sweep.
We also provide implementation details in the following sections. 


\subsection{IRPO}\label{app_subsec:IRPO}
When we optimize a base policy using the IRPO gradient defined in \Cref{eqn:IRPO_gradient}, we gradually anneal the temperature $\tau$ from 1 to 0 over the first 10\% of training time steps. 
This schedule encourages exploration of diverse gradients from exploratory policies in the initial stage of training, while encouraging exploitation as training converges.

Additionally, directly computing the policy gradient in \Cref{eqn:backpropagation_of_extrinsic_policy_gradient} is computationally expensive, as it requires computing the full Jacobian.
To address this, we leverage the vector-Jacobian product of automatic differentiation libraries such as \href{https://pytorch.org/tutorials/beginner/blitz/autograd_tutorial.html}{\texttt{PyTorch}} by leveraging the computational graph history across exploratory policy updates.
In turn, this reduces the computational complexity from $\mathcal{O}(m^2)$ to $\mathcal{O}(m)$ per backward pass, where $m$ is the number of policy parameters.

From \Cref{tab:algorithmic_parameters}, we adopt a relatively high learning rate for exploratory policy updates because the exploratory policy update performs a single-batch gradient update, unlike other algorithms that rely on multiple mini-batch updates. 
However, we use a relatively small target KL divergence of $1\text{e}-3$ for trust-region updates of the IRPO gradient, as it showed more stable convergence than a higher KL divergence threshold.

\begin{table}[t!]
    \centering
    \caption{Algorithm hyperparameters used for our experiments. General parameters apply to all algorithms. We show algorithm-specific hyperparameters for IRPO and baseline algorithms.}
    \label{tab:algorithmic_parameters}
    \begin{tabular}{lc|lcc}
        \toprule
        \multicolumn{2}{c|}{\textbf{General Parameters}} & \multicolumn{3}{c}{\textbf{Algorithm-Specific Parameters}} \\
        \cmidrule(r){1-2} \cmidrule(l){3-5}
        \textbf{Parameter} & \textbf{Value} & \textbf{Parameter} & \textbf{IRPO} & \textbf{Baseline} \\
        \midrule
        Actor hidden layers & $(64, 64)$ & Target KL divergence & $1 \times 10^{-3}$ & $1 \times 10^{-2}$ \\
        Critic hidden layers & $(128, 128)$ & Actor learning rate & $1 \times 10^{-2}$ & $3 \times 10^{-4}$ \\
        Activation & Tanh & Critic learning rate & $1 \times 10^{-3}$ & $3 \times 10^{-4}$ \\
        PPO clip ratio & $0.2$ & Discount factor & \multicolumn{2}{c}{listed in \Cref{tab:environmental_parameters}} \\
        \bottomrule
    \end{tabular}
\end{table}

\subsection{Hierarchical RL}
Our HRL baseline \cite{sutton1999between} pretrains subpolicies using intrinsic rewards (from ALLO \cite{gomez2023proper}) with PPO as the optimization algorithm.
That is, given $K$ intrinsic reward functions, we learn $K$ subpolicies with PPO, forming the following subpolicy set:
\begin{equation}
    \Pi = \{ \pi_{\tilde{\theta}_k} \mid k = 1,\dots, K \} \cup \{ \pi_{\text{rw}} \},
\end{equation}
where $\pi_{\text{rw}}$ denotes an additional random walk policy for a nearby fine-grained movement.
We then optimize a high-level policy $\pi_\theta$ whose action space is defined as $\mathcal{A} = \Pi$, selecting which subpolicy to activate.
Following a common hierarchical RL design, any subpolicies can be initiated over the full state space $\mathcal{S}$ and terminate either upon reaching their optima or after a fixed number of time steps, which was set to $10$ in our experiments.

\subsection{PSNE}
We follow the parameter noise injection implementation using TRPO as in~\cite{plappert2018parameter}.
We sample $K$ independent noises from a Gaussian distribution $ \epsilon_k \sim \mathcal{N}(0, \mb{I}) $ and inject these into the base policy parameters via re-parameterization as follows: $ \tilde{\theta}_k = \theta + \epsilon_k \sqrt{\Sigma} $, where $ \theta $ are the base policy parameters, $\epsilon_k$ is the $k$-th noise, and $ \Sigma $ is the learnable variance. 
This injection is performed in a way that constrains the KL divergence between the base and noise-injected policies using backtracking line search to prevent drastic updates. 
Rollouts are then collected using noise-injected policies $\{ \pi_{\tilde{\theta}_k}\}_{k=1}^K$, and both the base policy $ \pi_\theta $ and the variance $ \Sigma $ are updated accordingly based on the average performance of $\{ \pi_{\tilde{\theta}_k}\}_{k=1}^K$. For more details, we refer the reader to \cite{plappert2018parameter}.

\subsection{DRND}
Based on the PPO algorithm, DRND \cite{yang2024exploration} encourages exploration by incentivizing the visitation of states with high model uncertainty and low visitations. 
We directly adopt the \href{https://github.com/yk7333/DRND}{DRND codebase} with its default DRND hyperparameters (our environments are considered within the original DRND implementation), while keeping relevant PPO parameters as listed in \Cref{tab:algorithmic_parameters}.

\subsection{PPO}
We follow the standardized PPO implementation as described in \cite{shengyi2022the37implementation}.

\subsection{TRPO}
We directly adopt the TRPO algorithm \cite{schulman2015trust} from \href{https://github.com/ikostrikov/pytorch-trpo}{TRPO codebase}, while keeping relevant hyperparameters as listed in \Cref{tab:algorithmic_parameters}.

\begin{figure*}[t!]
    \centering
    \includegraphics[width=0.95\linewidth]{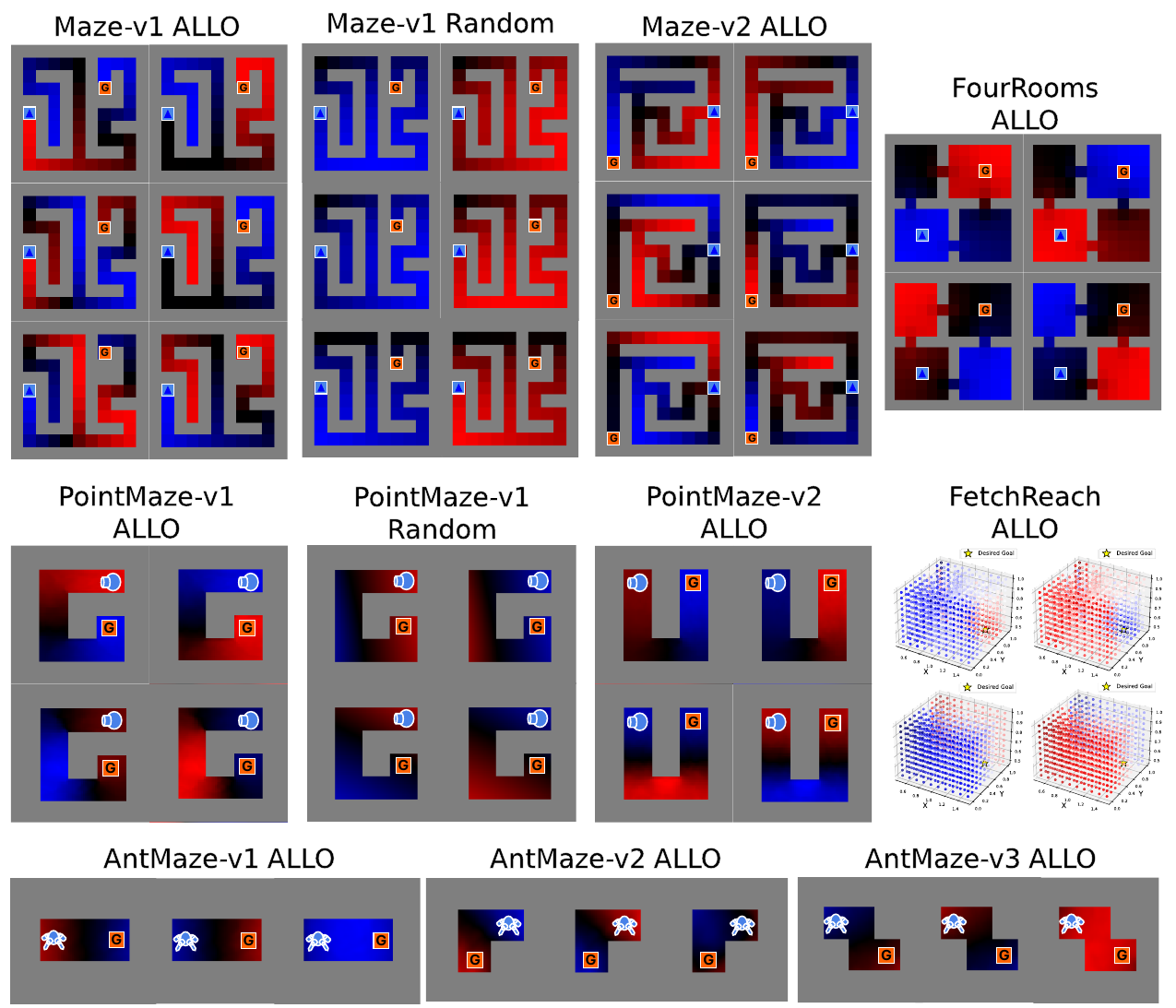}
    \caption{We show the intrinsic rewards, normalized to be in the range of $[-1, +1]$, generated by ALLO and use by IRPO-ALLO and HRL-ALLO. We also show one set of random intrinsic rewards used in our experiments. Color intensity represents the magnitude of the intrinsic reward: red indicates $+1$, black indicates $0$, and blue indicates $-1$.}
    \label{fig:visualization_of_ALLO}
\end{figure*}

\section{Intrinsic Rewards from ALLO} \label{app:intrinsic_reward_maps}
We use the \href{https://github.com/tarod13/laplacian_dual_dynamics}{ALLO codebase} for our experiments with its default hyperparameters \cite{gomez2023proper}. 
These include a discount sampling parameter of $0.9$ for discrete environments and $0.99$ for continuous ones.
Additionally, we employ a neural network with hidden layers of $(256, 256, 256, 256)$, a learning rate of $3\text{e}-4$, and ReLU activations. 
The Cartesian coordinates of the agent are used as an input, and the network is trained for $200,000$ epochs with batch size of $1024$. 
For details, we refer the reader to \cite{gomez2023proper}.
The resulting ALLO-based intrinsic rewards used in our experiments are visualized in \Cref{fig:visualization_of_ALLO}.
We used one set of intrinsic rewards for all experiments (IRPO-ALLO and HRL-ALLO) to reduce computation time---we saw negligible differences in intrinsic rewards across seeds in a limited set of experiments.

\section{Computational Complexity}\label{app:computational_time}
\begin{wrapfigure}{t}{0.35\textwidth}
    \centering
    \begin{tabular}{lcc}
        \toprule
        \textbf{Env.} & \textbf{IRPO} & \textbf{PPO} \\
        \midrule
        \emph{PointMaze-v1} & $2.21$ & $1.62$ \\
        \emph{FetchReach} & $1.12$ & $0.76$ \\
        \bottomrule
    \end{tabular}
    \caption{Mean wall-clock time over 10 random seeds in hours.}
    \label{tab:wall_clock}
\end{wrapfigure}

One potential concern may be the computational cost of IRPO. Specifically, IRPO requires learning a pair of critics for each intrinsic reward and computing the Jacobian, along with a trust-region update. 
Despite its computational cost, we show that the difference in mean wall-clock time compared to PPO is small in the \emph{PointMaze-v1} and \emph{FetchReach} environments. The mean wall-clock time was measured over 10 random seeds using an AMD Ryzen 9 7900X CPU, 32GB DDR5-6400 RAM, and an NVIDIA RTX 4070 SUPER GPU and is reported in \Cref{tab:wall_clock}.


\end{document}